\documentclass{article}

\usepackage{hyperref}

\usepackage{amsmath,amsfonts,bm}

\def\eqref#1{equation~\ref{#1}}

\def\1{\bm{1}}

\def\vh{{\bm{h}}}

\def\vv{{\bm{v}}}

\def\vx{{\bm{x}}}
\def\vy{{\bm{y}}}
\def\vz{{\bm{z}}}

\def\mI{{\bm{I}}}

\def\mL{{\bm{L}}}
\def\mM{{\bm{M}}}

\def\mR{{\bm{R}}}

\def\mU{{\bm{U}}}
\def\mV{{\bm{V}}}
\def\mW{{\bm{W}}}

\def\mSigma{{\bm{\Sigma}}}

\DeclareMathAlphabet{\mathsfit}{\encodingdefault}{\sfdefault}{m}{sl}
\SetMathAlphabet{\mathsfit}{bold}{\encodingdefault}{\sfdefault}{bx}{n}

\def\sD{{\mathbb{D}}}

\def\sR{{\mathbb{R}}}
\def\sS{{\mathbb{S}}}

\DeclareMathOperator*{\argmin}{arg\,min}

\usepackage[utf8]{inputenc} 
\usepackage[T1]{fontenc}    
\usepackage{hyperref}       
\usepackage{url}            
\usepackage{booktabs}       
\usepackage{amsfonts}       
\usepackage{nicefrac}       

\usepackage{microtype}      
\usepackage{multirow}
\usepackage[title]{appendix}
\usepackage{dsfont}
\usepackage{csquotes}
\usepackage{amsthm}
\usepackage{graphicx} 
\usepackage{tablefootnote}
\usepackage{thm-restate}
\usepackage{diagbox}
\usepackage{stfloats}

\usepackage{enumitem}

\usepackage{adjustbox}

\usepackage{thm-restate}

\usepackage{array}
\newcolumntype{P}[1]{>{\centering\arraybackslash}p{#1}}

\newtheorem{remark}{Remark}
\newtheorem{definition}{Definition}

\usepackage[accepted]{icml2021}

\icmltitlerunning{Large Scale Private Learning via Low-rank Reparametrization}

\begin{document}

\twocolumn[
\icmltitle{Large Scale Private Learning via Low-rank Reparametrization}




\begin{icmlauthorlist}
\icmlauthor{Da Yu}{to,goo}
\icmlauthor{Huishuai Zhang}{goo}
\icmlauthor{Wei Chen}{goo}
\icmlauthor{Jian Yin}{to}
\icmlauthor{Tie-Yan Liu}{goo}
\end{icmlauthorlist}

\icmlaffiliation{to}{The School of Data and Computer Science \& Guangdong Key Laboratory of Big Data Analysis and Processing, Sun Yat-sen University, Guangdong, China. The work was done when D. Yu was an intern at Microsoft Research Asia.}
\icmlaffiliation{goo}{Microsoft Research Asia, Beijing, China}

\icmlcorrespondingauthor{Wei Chen}{wche@microsoft.com}
\icmlcorrespondingauthor{Jian Yin}{issjyin@mail.sysu.edu.cn}

\icmlkeywords{Machine Learning, ICML}

\vskip 0.3in
]



\printAffiliationsAndNotice{}  

\begin{abstract}

We propose a reparametrization scheme to address the challenges of applying differentially private SGD  on large neural networks, which are 1) the huge memory cost of storing individual gradients, 2)  the added noise suffering notorious dimensional dependence.  Specifically, we reparametrize each weight matrix with two \emph{gradient-carrier} matrices  of small dimension and a \emph{residual weight} matrix. We argue that such reparametrization keeps the forward/backward process unchanged while enabling us to compute the  projected gradient  without computing the gradient itself. To learn with differential privacy, we design \emph{reparametrized gradient perturbation (RGP)} that perturbs the gradients on gradient-carrier matrices and reconstructs an update for the original weight from the noisy gradients. Importantly, we use historical updates to find the gradient-carrier matrices, whose optimality is rigorously justified under linear regression and empirically verified with deep learning tasks. RGP significantly reduces the memory cost and improves the utility. For example, we are the first able to apply 
differential privacy on the BERT model and achieve an average accuracy of $83.9\%$ on four downstream tasks  with $\epsilon=8$, which is within $5\%$ loss compared to the non-private baseline but enjoys much lower privacy leakage risk. 

\end{abstract}

\section{Introduction}

A recent line of works \citep{shokri2017membership,carlini2019secret,carlini2020extracting} have exposed the potential privacy risks of trained models, e.g., data extraction from language model. Theoretically, learning with  \emph{differential privacy} \citep{dwork2006calibrating} is guaranteed to prevent such information leakage because differential privacy imposes an upper bound on the influence of any individual sample. Empirically, differential privacy also makes learning more resistant to attacks \citep{rahman2018membership,bernau2019assessing, zhu2019deep, carlini2019secret, ma2019data,lecuyer2019certified}.

To learn with differential privacy, many  algorithms have been proposed under different settings over the past decade, e.g.,  \citet{chaudhuri2009privacy,song2013stochastic,agarwal2018cpsgd,wang02019differentially,wang2019differentially,yu2020gradient,phan2020scalable,vietri2020private}, to name a few. Among them, \emph{gradient perturbation} is a popular choice because of its simplicity and wide applicability \cite{abadi2016deep}. In terms of simplicity, gradient perturbation only makes two simple modifications to the standard learning process. It  first clips the gradients of individual samples, referred to as individual gradients, to bound the sensitivity and then perturbs the aggregated gradient with random noise. In terms of wide applicability, it does not assume the objective to be convex and hence applies to deep  neural networks.

  \begin{figure}
    \centering
  \includegraphics[width=1.0\linewidth]{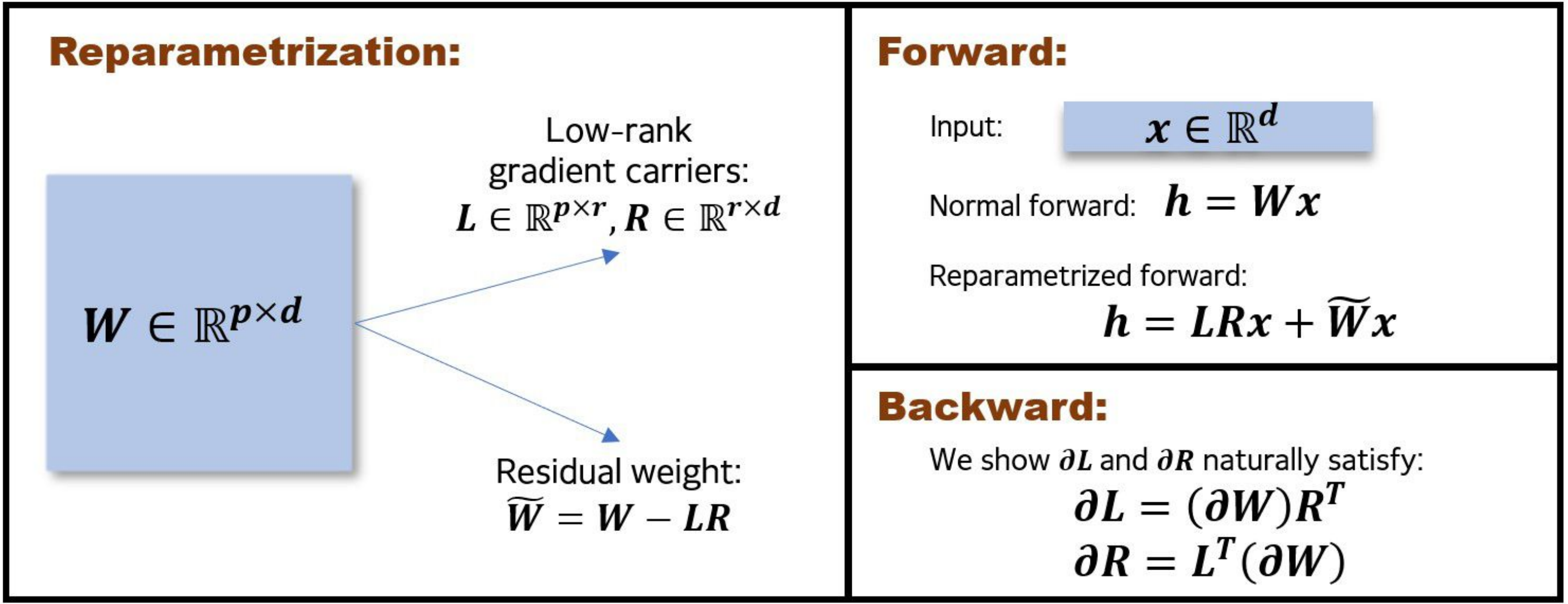}
  \caption{The proposed  reparametrization scheme.  The residual weight makes the reparametrized output the same as the normal output and  $\partial\mL$, $\partial\mR$ naturally connected with the normal gradient.    }
  \label{fig:repara}
\end{figure}

Despite its advantages, there are two challenges when applying gradient perturbation to cutting-edge deep models. First, one needs to compute and store individual gradients. Recent works \citep{dangel2019backpack,Opacus} have developed toolkits to compute individual gradients for a mini-batch of data through a single forward/backward pass, but storing individual gradients consumes a huge amount of memory as each individual gradient requires the same amount of memory as the model itself. Second, both theoretical and empirical utilities of gradient perturbation suffer from bad dependence on the model size \citep{bassily2014differentially, papernot2020tempered,tramer2021differentially} because the intensity of the added noise scales proportionally with the model size.

To tackle these challenges, we reparameterize each weight matrix $\mW$ of a deep neural network with a pair of low-rank \emph{gradient carriers} $\{\mL,\mR\}$ and a \emph{residual weight} $\tilde{\mW}$, as illustrated in Figure~\ref{fig:repara}. With this reparametrization, the forward signal and the  backward signal propagate the same as before. We show that the gradients on $\mL$ and $\mR$ are naturally connected with the  gradient on $\mW$.  Especially if the gradient carriers consist of orthonormal vectors,  we can construct a projection of the gradient of $\mW$ from the gradients of $\mL$ and $\mR$ that are of low dimension.  In other words, we can compute the projection of the gradient without computing the gradient itself. This property could save a huge amount of memory in DP-SGD where a large batch of individual gradients are  computed and stored.  We note that this could be also useful in other problems involving statistics of individual gradients, e.g. computing the gradient variance \citep{zhao2015stochastic,balles2016coupling,mahsereci2017probabilistic,balles2018dissecting}, which is out of our scope.

Based on the above framework, we propose \emph{reparametrized gradient perturbation (RGP)} for differentially private learning.  Specifically, after the backward process, RGP clips and perturbs the gradients of $\mL$ and $\mR$, which gives a certain level of privacy guarantee. Then RGP uses the noisy gradients to construct an update for the original weight. We note that because the gradient-carrier matrices are of much smaller dimension than the original weight matrix, the total intensity of the added noises is significantly smaller, which helps us break the notorious dimensional dependence of the utility of differentially private learning.

The key of the reparameterization scheme is how well the gradient projection approximates the original gradient. We argue that the approximation is good if 1) the original gradient of $\mW$ itself is indeed low-rank and 2) its principal subspace aligns with $\mL$ and $\mR$. The first condition is empirically verified by showing the gradient of each layer is of low stable rank when training deep neural networks, which has also been exploited for gradient compression in distributed optimization \citep{vogels2019powersgd}. The second condition is guaranteed if $\mL$ and $\mR$ consists of the principal singular vectors of the original gradient, which, however, violates the differential privacy. Instead, in RGP, we approximately compute a few of principal vectors of the historical updates that are already published and free to use because of the post-processing property of differential privacy, and use them as gradient carriers. We theoretically prove that the optimality of using the historical update substitution for linear regression and empirically verify its efficacy for deep neural networks.

With RGP, we can easily train large models with differential privacy and achieve good utility on both the vision and language modeling tasks. For example, we use RGP to train the BERT model \citep{devlin2018bert} on downstream language understanding tasks. We establish rigorous differential privacy guarantee for such large model with a modest drop in accuracy. With a privacy budget $\epsilon=8$, we achieve an average accuracy $83.9\%$ on downstream tasks, which is within $5\%$ loss compared to the non-private baseline. We also use \emph{membership inference attack} \citep{shokri2017membership,sablayrolles2019white} to evaluate the empirical privacy risks and demonstrate that the models trained with RGP are significantly more robust to membership inference attack than the non-private ones.
 Overall, our contribution can be summarized as follows.
\begin{enumerate}[itemsep=0mm]
\item We propose reparametrized gradient perturbation (RGP) that reduces the memory cost and improves the utility when applying DP on large models.

\item We give a detailed analysis on the property of RGP. We propose using the historical update to find the principal subspace and give theoretical arguments.

\item Empirically we are able to efficiently train BERT with differential privacy on downstream tasks, and achieve both good accuracy and privacy protection.
\end{enumerate}

\subsection{Notations}

We introduce some basic notations. Vectors and matrices are denoted with bold lowercase letters, e.g., $\vv$, and bold capital letters, e.g., $\mM$, respectively. Sets are denoted with double-struck capital letters, e.g., $\sS$. We use $[n]$ to denote the set of positive numbers $\{1,...,n\}$. Some preliminaries on differential privacy are presented in Appendix \ref{app:sec:preliminary}.

\section{A Reparametrization Scheme}\label{sec:lrk}

In this section, we introduce a reparametrization scheme for the neural network weight matrices so that computing and storing individual gradients are affordable for large models. Specifically,  during each forward/backward process, for a layer with weight matrix $\mW\in \sR^{p\times d}$,  we reparametrize it as follows (see  Figure~\ref{fig:repara} for an illustration),
\begin{flalign} 
\mW \rightarrow \mL \mR + \tilde{\mW}.{stop\_gradient()}, \label{eq:repara}
\end{flalign}
where $\mL\in\sR^{p\times r}, \mR\in\sR^{r\times d}$ are two low-rank gradient carriers  with $r\ll p \text{ or } d$, $\tilde{\mW} = \mW-\mL\mR$ represents the residual weight and $.{stop\_gradient()}$ means that we do not collect the gradient on $\tilde{\mW}$. 
Hence, such reparametrization does not change the forward signal and the backward signal, but only changes the gradient computation. Now we obtain the gradients on $\mL$ and $\mR$. We then unveil the connection between the gradient on $\mW$ and the gradients on $\mL$ and $\mR$.

\begin{restatable}{theorem}{gradlr}\label{thm:grad_lr}
For a layer with weight matrix $\mW$, suppose that $\partial\mW$ is the gradient computed by back-propagation with a mini-batch data $\sD$. Given two matrices $\mL, \mR$, we reparametrize $\mW$ as in Eq~(\ref{eq:repara}) and compute the gradients $\partial\mL$ and  $\partial\mR$  by running the forward and backward process with the same mini-batch  $\sD$, then 
\begin{flalign}
\partial \mL = (\partial\mW)\mR^{T},\;\; \partial \mR = \mL^{T}(\partial\mW).
\end{flalign}

\end{restatable}

Based on the above understanding, we can construct an update for $\mW$ by using $\partial\mL$ and $\partial\mR$. 

\begin{restatable}{corollary}{corogradlr}\label{corollary:grad_lr}
If the columns of $\mL$ and the rows of $\mR$ are orthonormal, respectively, and we use 
\begin{flalign}
\label{eq:grad_lrk}
(\partial \mL) \mR + \mL(\partial\mR) - \mL\mL^{T}(\partial\mL)\mR,
\end{flalign}
as the update for $\mW$, then the update is equivalent to projecting $\partial\mW$ into the subspace of matrices whose row/column spaces are spanned by $\mL$ and $\mR$.
\end{restatable}

\begin{proof}
The proofs of Theorem~\ref{thm:grad_lr} and Corollary~\ref{corollary:grad_lr} are relegated to Appendix~\ref{apd:subsec:proof_sec2}.
\end{proof}

We note that if $\mL$ and $\mR$ consist of orthonormal bases, Corollary~\ref{corollary:grad_lr} states that we can obtain the projection of  $\partial\mW$ without explicitly computing and storing $\partial\mW$! The size of gradient on $\mL$ or $\mR$ is much smaller than the size of $\partial\mW$ if the gradient carriers are chosen to be low-rank. Therefore, this reparametrization provides a convenient way to compute and store projected gradients of a large matrix. This is extremely beneficial for the scenarios where individual gradients $\{\partial_i \mW\}_{i=1}^{m}$ are required, e.g., approximating the variance of gradients and controlling the gradient sensitivity.

It is natural to ask how to choose $\mL$ and $\mR$ so that the update in Corollary~\ref{corollary:grad_lr} contains the most information of $\partial\mW$. Ideally, we can first compute the aggregated gradient $\partial\mW$ and run \emph{singular value decomposition} (SVD) $\partial\mW=\mU\mSigma\mV^{T}$. Then we can choose the top few columns of $\mU$ and $\mV$ to serve as the gradient carriers. In this case, the update in Corollary~\ref{corollary:grad_lr} is equivalent to approximating $\partial\mW$ with its top-$r$ principal components.

However, in the context of differential privacy, we can not directly decompose $\partial \mW$  as it is private. In the sequel, we give a practical reparametrization scheme for differentially private learning, where we use the historical update to find $\mL$ and $\mR$ and argue the optimality under certain conditions.

One may wonder why not just replace $\mW$ with $\mL$ and $\mR$ instead of doing the reparametrization. We note that the forward and the backward process remain the same as before if doing the reparametrization, and the only change is the gradient computation of $\mW$. In contrast, if using $\mL$ and $\mR$ to replace the weight $\mW$, this would not only reduce the expressive power but also hurt the optimization as the width varies dramatically across layers and the forward/backward signals cannot propagate well by common initialization strategies \cite{glorot2010understanding, he2016deep}.

\subsection{Reparametrization for Convolutional Layers}
\label{sec:lrk_conv}
In the above, we have described how to reparametrize a weight matrix, which covers the usual fully-connected layer and the attention layer in language models. In this subsection, we show the reparametrization of convolutional layers.  Let $\vx\in\mathbb{R}^{d\times w' \times h'}$ be the input feature maps of one sample and $\vh\in\mathbb{R}^{p\times w \times h}$ be the output feature maps. We describe how to compute the elements at one spatial position $\vh_{:,i,j}\in\mathbb{R}^{p}$ where $i\in [0,w]$ and $j\in [0,h]$.

Let $\mW\in \mathbb{R}^{p\times d\times k\times k}$ be the convolution kernels and $\vx^{(i,j)}\in\mathbb{R}^{d\times k\times k}$ be the features that we need to compute $\vh_{:,i,j}$. The output feature $\vh_{:,i,j}$ can be computed as 
$\vh_{:,i,j}=\bar\mW \vx^{(i,j)}$,
where $\bar\mW\in\mathbb{R}^{p\times dk^{2}}$ is obtained by flattening the channel and kernel dimensions. Hence, we can use the same way as in Eq~(\ref{eq:repara}) to  reparametrize $\bar\mW$: 
\begin{flalign}
\vh_{:,i,j} = \mL\mR\vx^{(i,j)} + (\bar\mW-\mL\mR)\vx^{(i,j)}.
\end{flalign}

Specifically, the operation of $\mR$ and $\mL$ are implemented by  two consequent convolutional layers with kernel sizes $r\times d\times k\times k$ \ and $p\times r\times 1\times 1$, respectively, where $r$ is the reparametrization rank. The residual weight is implemented by a convolutional layer of the original kernel size.

\section{Private Deep Learning with Reparametrized Gradient Perturbation}
\label{sec:dp_learning_lrk}

The above reparametrization strategy can significantly reduce the gradient dimension, which could help us circumvent the difficulties of applying differential privacy on large machine learning models. In this section, we propose a procedure ``reparametrized gradient perturbation (RGP)'' to train large neural network models  with differential privacy. Specifically, Section \ref{subsec:dp_learning_lrk_algo} introduces the whole procedure of RGP, Section \ref{subsec:privacy_rgp} gives the privacy guarantee of RGP, and Section \ref{subsec:complexity} presents the complexity analysis.

\begin{algorithm}[tb]
   \caption{Reparametrized Gradient Perturbation (RGP)}
   \label{alg:dp_lrk_repara}
\begin{algorithmic}[1]
  \STATE {\bfseries Input:} NN with weight matrices $\{\mW^{(l)}\}_{l=1}^{H}$,  steps $T$,  probability $q$, variance $\sigma^2$, clipping threshold $C$, warm-up steps $T_{\text{warm-up}}$, Algorithm \ref{alg:decompose_pi} input $\{r, K\}$.

  \STATE Randomly initialize the weights and obtain $\{\mW^{(l)}_{0}\}_{l=1}^H$;
  \FOR{$t=1$ {\bfseries to} $T$}
    \medskip
        \STATE Sample a minibatch $\{\vx_{i}\}_{i\in S_t}$ with probability $q$;
    \medskip

      \STATE For all $l\in [H]$, compute historical updates 
      $$\Delta_t^{(l)} \leftarrow \mW^{(l)}_t - \mW^{(l)}_0 \cdot 1_{\{t>T_{\text{warm-up}}\}};$$
      and run Alg.~\ref{alg:decompose_pi} with $\{\Delta_t^{(l)}, r, K\}$ to get $\mL_t^{(l)},\mR_t^{(l)}$;

    \medskip
    \STATE \textsl{//Forward/backward process with reparametrization.}

    \STATE Run reparametrized forward process with Eq~(\ref{eq:repara});
    \STATE Run backward process and compute individual gradients $\{\partial_i\mL_t^{(l)},\partial_i\mR_t^{(l)}\}_{l\in[H], i\in S_t}$;
    \medskip
    \STATE \textsl{//Bound gradient sensitivity and add noise.}
    
    \STATE Clip individual gradients with $L_{2}$ norm threshold $C$;
   
        \FOR{$l=1$ {\bfseries to} $H$} 
    
    \STATE Sum individual gradients and get  $\{\partial\mL_t^{(l)},\partial\mR_t^{(l)}\}$;
    \STATE Perturbation with Gaussian noise $\vz_{L,t}^{(l)},\vz_{R,t}^{(l)}$ whose elements are independently from $\mathcal{N}(0,\sigma^{2}C^{2})$:
       $$\tilde{\partial}\mL_t^{(l)}  \leftarrow \partial\mL_t^{(l)} + \vz_{L,t}^{{(l)}}, \quad \tilde{\partial}\mR_t^{(l)} \leftarrow \partial\mR_t^{(l)}+\vz_{R,t}^{(l)};$$ 
    \STATE Use $\tilde{\partial}\mL_t^{(l)}$, $\tilde{\partial}\mR_t^{(l)}$, and Eq~(\ref{eq:grad_lrk}) to  construct $\tilde{\partial} \mW_t^{(l)}$;
    \STATE Use off-the-shelf optimizer to get $\mW_{t+1}^{(l)}$;
    \ENDFOR

  \ENDFOR

\end{algorithmic}
\end{algorithm}

\subsection{Reparametrized Gradient Perturbation Algorithm}
\label{subsec:dp_learning_lrk_algo}

The pseudocode of RGP is presented in Algorithm~\ref{alg:dp_lrk_repara}. The RGP proceeds for all the layers and we ignore the layer index for simplicity in the following discussion. At each update, for a layer with weight matrix $\mW$, RGP consists of four steps: 1) generate the gradient-carrier matrices $\mL$ and $\mR$, 2) run the reparametrized forward/backward process and obtain the individual gradients $\{\partial_i \mL\}_{i=1}^{m}$ and $\{\partial_i \mR\}_{i=1}^{m}$, 3) clip and perturb the gradients, 4) reconstruct an approximated gradient on the original weight matrix. 

In the RGP procedure, \textbf{step 1)}, which is also the core challenge, is to choose ``good" gradient-carrier matrices so that the reconstructed gradient can approximate the original gradient as well as possible. First, this requires for a given rank $r$, the generated gradient-carrier matrices should align with the principal components of the original gradient well. Moreover, to reconstruct the gradient in step 4), it requires the gradient carriers have orthonormal columns/rows. 

For the first requirement, we use  historical updates to find the gradient carriers. The historical update is not sensitive because of the post-processing property of differential privacy.  In Section~\ref{subsec:historical_grad}, we give both empirical and theoretical arguments to  demonstrate that the principal subspace of the current gradient aligns with that of the historical update.   In our implementation, we use a warm-up phase in which the decomposition is directly done on the weight. We approximate the principal components via the power method (Algorithm~\ref{alg:decompose_pi}) instead of  the time-consuming full SVD. For the second requirement, we apply  the Gram-Schmidt process to orthonormalize $\mL$ and $\mR$.

\begin{algorithm}[tb]
   \caption{Decomposition via Power Method.}
   \label{alg:decompose_pi}
\begin{algorithmic}
  \STATE {\bfseries Input:} Historical update $\Delta$, reparametrization rank $r$, number of iterations $K$.
  \STATE {\bfseries Output:} Gradient carriers $\mL\in\mathbb{R}^{p\times r}$,  $\mR\in\mathbb{R}^{r\times d}$.
  
  \medskip

    \STATE Initialize  $\mR$ from standard Gaussian distribution.
  
    \FOR{$k=1$ {\bfseries to} $K$} 
    \STATE $\mL \leftarrow \Delta \mR^{T}$
    \STATE Orthonormalize the columns of $\mL$.
    \STATE $\mR=\mL^{T}\Delta$
    \ENDFOR
    \STATE Orthonormalize the rows of $\mR$.

    \STATE Return $\mL$, $\mR$
\end{algorithmic}
\end{algorithm}

\textbf{Step 2)} of RGP is the reparametrization and a round of forward/backward propagations, as presented in Section \ref{sec:lrk}.

\textbf{Step 3)} is for differential privacy guarantee. The individual gradients $\{\partial_i \mL, \partial_i \mR\}_{i=1}^{m}$ are first clipped by a pre-defined threshold so that the sensitivity is bounded. Then, Gaussian noise is added to the aggregated gradient to establish a differential privacy bound. The energy of added noise is proportional to the dimension, i.e., the rank $r$ of the carrier matrices. Hence, in order to make the noise energy small, it encourages us to use smaller rank $r$. However, smaller rank would increase the approximation error in the \textbf{step 1)}. In practice, we trade off these two factors to choose a proper $r$.

In \textbf{step 4)}, we use the noisy aggregated gradients of gradient-carrier matrices to reconstruct the gradients of original weights, as depicted in Corollary~\ref{corollary:grad_lr}. The reconstructed gradients can then be used by any off-the-shelf optimizer.

\subsection{Privacy Analysis of RGP}
\label{subsec:privacy_rgp}

The privacy bound of Algorithm~\ref{alg:dp_lrk_repara} is given by  Proposition~\ref{prop:privacy}. The derivation of Proposition~\ref{prop:privacy} is based on the \emph{moments accountant} that is proposed in \citet{abadi2016deep}.  Moments accountant has tighter composition bound than the strong composition theorem in \citet{algofound}.  Moments accountant first tracks the privacy budget spent at each update. Then, it composes the spent budget of all updates and cast the final privacy cost into the classic $(\epsilon,\delta)$-differential privacy.

\begin{restatable}[\citet{abadi2016deep}]{proposition}{privacy}\label{prop:privacy}
There exist constants $c_1$ and $c_2$ so that given running steps $T$, for any $\epsilon<c_{1}q^{2}T$, Algorithm~\ref{alg:dp_lrk_repara} is $\left(\epsilon,\delta\right)$-differentially private for any $\delta>0$ if we choose \[\sigma\geq c_2\frac{q\sqrt{Tlog\left(1/\delta\right)}}{\epsilon}.\]
\end{restatable}

\begin{proof}
The proof outline is relegated to Appendix~\ref{apd:subsec:proof_sec3}.
\end{proof}

  The value of $\sigma$ in Proposition~\ref{prop:privacy} is based on an asymptotic bound on the moments of the privacy loss random variable. In practice, one can use the numerical tools \citep{wang2019subsampled,mironov2019renyi} to compute a tighter bound. So far we have depicted the overall picture of RGP. We next analyze  the computational and memory costs of RGP and compare them with that of DP-SGD.

\subsection{Complexity Analysis of RGP}
\label{subsec:complexity}
For the simplicity of notations, we only give the costs of one fully connected layer at one update (including forward and backward) and assume that the weight matrix is square. The shape of weight matrix, size of minibatch, number of power iterations, and rank of reparametrization are denoted by $(d\times d)$, $m$, $K$, and $r$, respectively.

The computational overhead of RGP consists of three parts. The first part is induced by matrix multiplication of power iteration, whose complexity is $\mathcal{O}(Krd^{2})$. The second part is induced by the Gram–Schmidt process, whose complexity is $\mathcal{O}(Kr^{2}d)$. The third part of overhead is the computational cost induced by gradient carriers during the forward/backward process, which is on the order of $\mathcal{O}(mrd)$. 

RGP uses much less memory than DP-SGD in the practice.   Although RGP needs some extra memory to store the activation produced by the gradient carriers, it has a significant advantage over DP-SGD on the memory cost of storing individual gradients, which is one of the main challenges of learning with differential privacy. For RGP, the memory cost of individual gradients only scales linearly with model width $d$ in contrast with $d^2$ for DP-SGD. We summarize the computational cost of one update and the memory cost of storing individual gradients in  Table~\ref{tbl:complexity}.

\begin{table} 
    \caption{Computation and memory costs of RGP (Algorithm~\ref{alg:dp_lrk_repara}) and DP-SGD \citep{abadi2016deep}, where $m$ is the size of mini-batch, $d$ is the model width, $r$ is the reparametrization rank, and $K$ is the number of power iterations.}
\label{tbl:complexity}
\centering
\small
\renewcommand{\arraystretch}{1.85}
\begin{tabular}{ P{2.45cm}|P{1.15cm}|P{3.4cm} }
 \hline \hline
   \backslashbox{Cost}{Method}            & DP-SGD 	& RGP  		 \\
 \hline
Computational cost     &   $\mathcal{O}(md^{2})$     & $\mathcal{O}(md^{2}+Krd^2+Kr^{2}d)$			\\\hline
Memory cost    &  $\mathcal{O}(md^{2})$    & $\mathcal{O}(mrd)$ 			\\

 \hline
 \hline
\end{tabular} 
\end{table}

The low-rank nature of gradient permits us to choose a small $r$ without destroying utility (see Section~\ref{subsec:grad_is_lrk}). In practice, we typically choose the rank $r$ smaller than $10$.  For the number of power iterations in Algorithm~\ref{alg:decompose_pi}, we find that setting $K=1$ is sufficient to get good performance. Hence, in practice, we always choose small $r$ and $K$ for efficiency while not hurting the performance.

\section{Two Properties of the Gradient Matrix} \label{sec:grad_property}

We show two properties of the gradients of modern deep neural networks to justify the design choices of Algorithm~\ref{alg:dp_lrk_repara}. The first property is that the gradient of each weight matrix is naturally low-rank, which motivates us to use low-rank reparameterization. The second property is that the gradient of a weight matrix along the optimization path could stay in the same subspace, which motivates us to use the historical updates to generate the gradient-carrier matrices.

\subsection{Gradient Matrix Is of Low Stable Rank}
\label{subsec:grad_is_lrk}

Recent works have used the low-rank approximation to compress the gradients and reduce the communication cost in distributed optimization \citep{yurtsever2017sketchy, wang2018atomo, karimireddy2019error, vogels2019powersgd}. These existing works set up a good motivation to exploit the low stable rank property of the gradients of weight matrices.

We further verify this low-rank property which may give a hint about how to set the reparameterization rank $r$ in practice. We empirically compute the stable rank ($\|\cdot\|_F^2/\|\cdot\|^2_{2}$) of the gradient of the weight matrices in a BERT model and a wide ResNet model. The dataset for the BERT model is SST-2 from the GLUE benchmark \citep{wang2018glue}. The dataset for the wide ResNet model is CIFAR-10 \cite{cifar}. The experimental setup can be found in Section~\ref{sec:exp}. We plot the gradient stable rank in Figure~\ref{fig:stbl_rank}. 

 \begin{figure}
    \centering
  \includegraphics[width=0.8\linewidth]{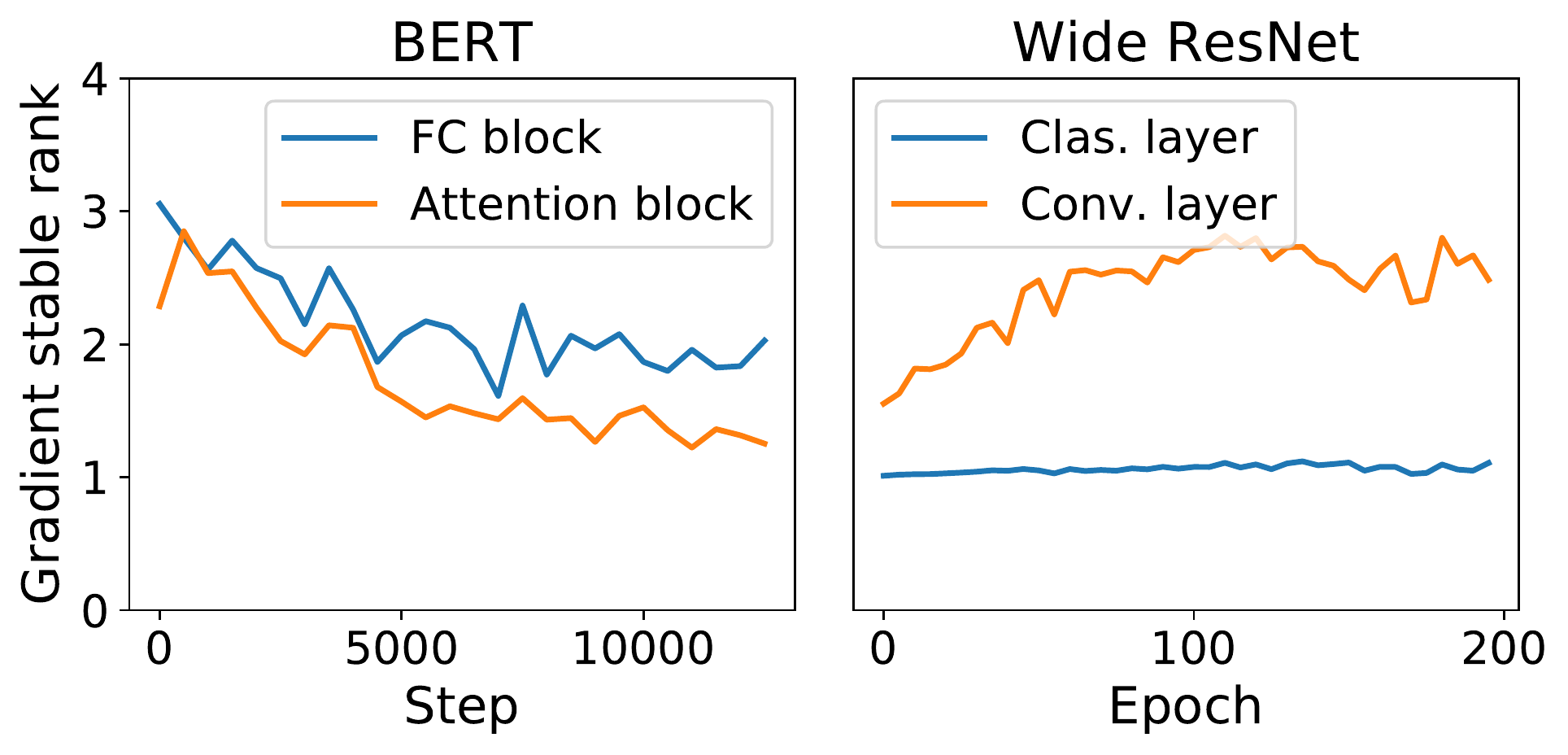}
  \caption{Gradient stable rank ($\|\cdot\|_F^2/\|\cdot\|^2_{2}$). For ResNet, we plot the gradient rank of the classification layer and the first residual block. For BERT, we plot the gradient rank of the first fully-connected block and the first attention block.}
  \label{fig:stbl_rank}
\end{figure}

As shown in Figure~\ref{fig:stbl_rank}, both the gradients of BERT and ResNet models are naturally of low stable rank over the training process. Hence, low-rank gradient-carrier matrices would have a small approximation error if  we find the right gradient subspace. In Section \ref{subsec:historical_grad}, we argue that historical update is a good choice to identify the gradient subspace.

\subsection{Historical Gradients Are Correlated}
\label{subsec:historical_grad}

Suppose that $\mW_t$ is a weight matrix at step $t$, and $\partial \mW_t$ is the gradient with a batch of data $\sD$ with a $r$-SVD $\partial \mW_t = \mU_t \Sigma_t \mV_t^T$. For another step ${t'}$ with $t'>t$ and the same data $\sD$, we have $\mW_{t'}, \partial \mW_{t'}$ and a $r$-SVD: $\partial \mW_{t'} = \mU_{t'} \Sigma_{t'} \mV_{t'}^T$. We can project $\partial \mW_{t'}$ onto the principal subspace of $\partial \mW_t$ or $\partial \mW_{t'}$ and measure the  projection residual
\begin{flalign}
&\|(\mI - \mU_t\mU_t^T)\partial \mW_{t'}(\mI-\mV_t\mV_t^T)\|_F/\|\partial \mW_{t'}\|_F,\label{eq:proj_res}
\\
&\|(\mI - \mU_{t'}\mU_{t'}^T)\partial \mW_{t'}(\mI-\mV_{t'}\mV_{t'}^T)\|_F/\|\partial \mW_{t'}\|_F,\label{eq:self_proj_res}
\end{flalign}
where Eq~(\ref{eq:proj_res}) is the projection residual using historical gradient, referred to as \emph{historical projection residual}, and Eq~(\ref{eq:self_proj_res}) is the projection residual using current gradient, referred to as \emph{self projection residual}. A small difference between Eq~(\ref{eq:proj_res}) and~(\ref{eq:self_proj_res}) indicates that the principal subspace of the current gradient aligns with that of the historical gradient.

 \begin{figure}[t]
    \centering
  \includegraphics[width=0.8\linewidth]{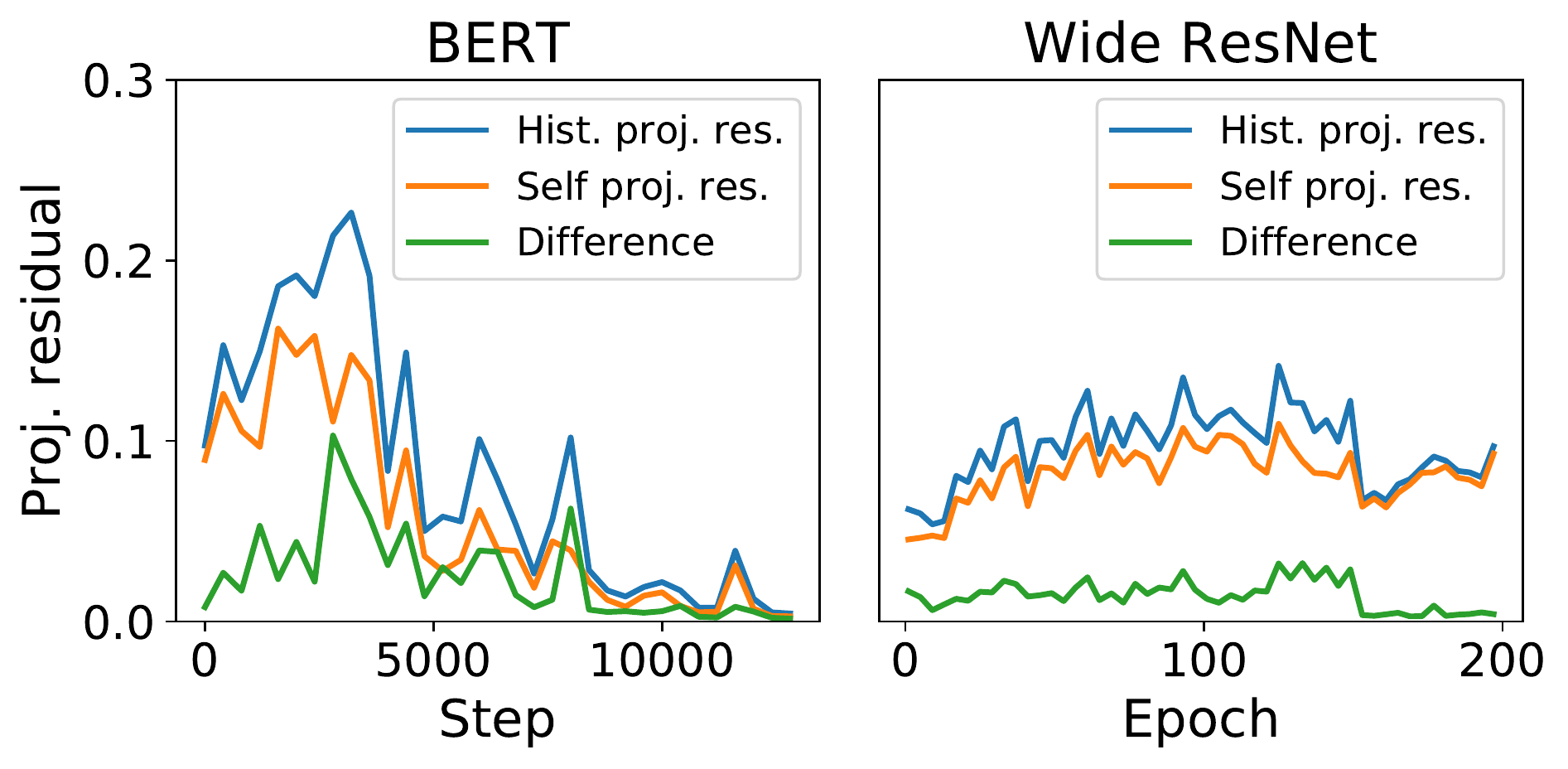}
  \caption{Projection residual with reparametrization rank $8$.  We use a fixed mini-batch with $500$ samples. For ResNet, we use the input convolution layer.  For BERT, we use the second matrix of the FC layer in the first encoder block. The definition of historical/self projection residual is in Eq~(\ref{eq:proj_res}) and~(\ref{eq:self_proj_res}).   }
  \label{fig:proj_res}
\end{figure}

We empirically examine the projection residual of a BERT model and a wide ResNet model. The tasks are the same as in Section~\ref{subsec:grad_is_lrk}.   At the beginning of each epoch, we evaluate the projection residual between the current gradient and the gradient of the previous epoch. The results are plotted in Figure~\ref{fig:proj_res}. We can see that the difference between  Eq~(\ref{eq:proj_res}) and~(\ref{eq:self_proj_res}) is small for both models.

To understand why historical gradients are correlated, we next use a linear regression problem to rigorously show that the gradients over time could live in the same subspace. Suppose we have a set of observations $\{(\vx_i, \vy_i)\}_{i=1}^n$, where $\vx_i \in \sR^d$ is the feature vector and $\vy_i\in \sR^{p}$ is the target vector for all $i \in [n]$. The least-squares problem is given by 
\begin{flalign}
\argmin_\mW  \frac{1}{n}\sum_{i=1}^n \|\vy_i - \mW \vx_i\|^2. \label{eq:least-squares}
\end{flalign}

\begin{restatable}{proposition}{gradalign}\label{prop:grad-align}
For the least squares problem (\ref{eq:least-squares}), if the model is updated by gradient descent with step size $\eta$
\begin{flalign}
\mW_{t+1} \leftarrow \mW_t - \eta \cdot \partial\mW_t, \label{eq:gd}
\end{flalign}
then the gradients $\{\partial \mW_t\}_{t\ge 1}$ share the same range and null space. That is to say, if $\partial \mW_1$ is rank $r$ and has $r$-SVD $\partial \mW_1 = \mU_1 \Sigma_1 \mV_1^T$, then for all $t\ge 1$, we have
\begin{flalign}
(\mI - \mU_1\mU_1^T) \partial \mW_t = 0,\;  \partial \mW_t (\mI - \mV_1\mV_1^T)= 0.
\end{flalign}
\end{restatable}

\begin{proof}
The proof is relegated to Appendix \ref{apd:subsec:proof_sec4}.
\end{proof}

Hence we can use the historical updates $\mW_t-\mW_0$ to identify  gradient row/column subspaces as in Algorithm \ref{alg:dp_lrk_repara}.

That indicates that for the weight matrix $\mW\in \sR^{p \times d}$, if the gradient turns out to be low-rank $r$ due to the data $\{\vx_i, \vy_i\}$, we can possibly first identify the intrinsic subspace which is of $r(p+d)$ dimension instead of the original $p\cdot d$ number of parameters. Then we can work within this subspace for differentially private empirical risk minimization. This can both reduce the effect of noise and save the memory cost of gradient perturbation due to the small intrinsic dimension.  
We note that identifying the low-rank subspace can be done approximately as in the algorithm, or by using some auxiliary public data as in \citet{zhou2021bypassing, yu2021do}.

\begin{remark}\label{rem:lst-sqr}
Suppose that the least-squares objective $L(\mW):=\frac{1}{n}\sum_{i=1}^n \|\vy_i - \mW \vx_i\|^2$ is $\beta$-smooth and the gradient subspace is rank $r$ and can be exactly identified. Let the optimizer of RGP be gradient descent and $\sigma$ be set as in Proposition~\ref{prop:privacy}. If $\eta=\frac{1}{\beta}$, $T=\frac{n\beta\epsilon}{\sqrt{p}}$, and $\bar{\mW}=\frac{1}{T}\sum_{t=1}^{T}\mW_t$, then 
\[\mathbb{E}[L(\bar{\mW})]-L(\mW_*)\leq \mathcal{O}\left(\frac{\sqrt{(p+d)r\log(1/\delta)}}{n\epsilon}\right),\]
where $\mW_*$ is the optimal point,  $\mW_{t}$ is the output of Algorithm~\ref{alg:dp_lrk_repara} at step $t$.
\end{remark}
The proof of Remark \ref{rem:lst-sqr} can be adapted from \cite{yu2020gradient}.  Although  the exact low-rank property of the gradient  cannot be  rigorously proved for deep neural network because of the co-adaption across layers, we have empirically verified that the gradient matrices are still of low stable rank and stay in roughly the  same subspace over iterations (see Figure \ref{fig:stbl_rank} \& \ref{fig:proj_res}). Our algorithm exploits this fact to reparameterize weight matrices, which achieves better utility and reduces the memory cost compared with DP-SGD.

\section{Experiments} \label{sec:exp}

We conduct experiments on various kinds of tasks to demonstrate the effectiveness of RGP. We first examine the utility of  models trained by RGP. To this end, we apply RGP  on the wide ResNet \citep{zagoruyko2016wide}  and  the BERT \citep{devlin2018bert} models, which are representative models for  computer vision and natural language modeling. The results are presented in Section~\ref{subsec:exp_resnet} and~\ref{subsec:exp_bert}.  The source code of our implementation is publicly available\footnote{\url{https://github.com/dayu11/Differentially-Private-Deep-Learning}}.

\begin{table}
\small
\renewcommand{\arraystretch}{1.2}
\centering
    \caption{Validation accuracy (in \%) of WRN28-4 on vision tasks . } \label{tbl:tbl_resnet}
    
        \begin{tabular}{l|l|l}
        \hline
        \hline
            Method         & SVHN &  CIFAR10    \\
            \hline
              Full (N.P.) &      97.2        &      93.3     \\\cline{1-3}
             Linear (N.P.) &      41.1        &     39.8            \\\cline{1-3}
             RGP (N.P.)   &        97.1       &    91.2      \\\cline{1-3}
             PowerSGD (N.P.) &      97.1         &        91.9         \\\cline{1-3}
                DP-SGD ($\epsilon=8$)  &      91.6         &       55.9   \\\cline{1-3}
             DP-PowerSGD ($\epsilon=8$) &         91.9      &     57.1           \\\cline{1-3}
                RGP-random ($\epsilon=8$)  &    91.7           &   51.0      \\\cline{1-3}
             RGP ($\epsilon=8$)&     94.2         &    63.4     \\\hline \hline
        \end{tabular}
\end{table}

\begin{table}
\small
\centering
    \caption{Validation accuracy (in \%) of RGP on vision tasks with varying $\epsilon$. The model architecture is WRN28-4. Numbers in  brackets denote the improvements compared to DP-SGD. } \label{tbl:vision_vary_eps}
    
        \begin{tabular}{l|l|l|l}
        \hline
        \hline
            Dataset         & $\epsilon=2$ &  $\epsilon=4$ & $\epsilon=6$    \\\hline
             SVHN &    87.3  (+4.1)      &  89.7   (+3.4)    & 92.3  (+3.9)  \\\hline
             CIFAR10 &     44.0 (+6.6)        &   53.3 (+6.4)   &     59.6 (+7.9)      \\\hline \hline
        \end{tabular}
\end{table}

Moreover, we empirically  evaluate the privacy risk of the models via the success rate of \emph{membership inference (MI) attack} \citep{shokri2017membership,sablayrolles2019white,yu2021how}. The results are  presented in Section~\ref{subsec:exp_mi}.

\textbf{Implementation.}  The number of iterations for power method is $1$. We use an open-source tool of moments accountant to compute the privacy loss\footnote{\url{https://github.com/tensorflow/privacy}}. For a given setting of hyperparameters, we  set $\sigma$ to be the smallest value so that the privacy budget is allowable to run desired epochs. All experiments are run on a node with four Tesla V100 GPUs.

\textbf{Baselines.} We implement several baseline algorithms for comparison.  For differentially private learning, the first baseline is \emph{DP-SGD} in \citet{abadi2016deep} and the second one is RGP with gradient carriers consisting of random orthonormal vectors, referred to as \emph{RGP-random}. We also include several non-private baselines, i.e., \textbf{(\romannumeral 1)} \emph{Full (N.P.)}: training the full model, \textbf{(\romannumeral 2)} \emph{Linear (N.P.)}: training only the linear classification layer, \textbf{(\romannumeral 3)} \emph{RGP (N.P.)}: training the model with reparametrization  but without gradient clipping or adding noise.

We consider differentially private  \emph{PowerSGD} \citep{vogels2019powersgd} as another baseline for vision tasks. PowerSGD approximates full gradients with low-rank matrices to reduce the communication cost. It first aggregates the individual gradients and then runs power iterations to find approximations of the principle components of the averaged gradient. Hence for DP-powerSGD, it is necessary to first perturb the aggregated gradient  and then project it into low-rank subspace otherwise  the sensitivity is hard to track after projection. As a consequence,  DP-powerSGD needs to compute the individual gradients explicitly, which costs huge memory as DP-SGD does. In Section~\ref{subsec:exp_resnet}, we add a DP-powerSGD baseline  with the same setting as that of RGP.

Additionally, some ablation experiments are conducted to study the influence of the residual weight and reparametrization ranks, which are relegated to the Appendix~\ref{app:sec:add-exp}.

\subsection{Experiments on Vision Tasks}\label{subsec:exp_resnet}

\textbf{Model.} We use wide ResNet models \citep{zagoruyko2016wide} for the vision tasks. The architecture is WRN28-4 with $\sim$1.5M parameters. All batch normalization layers are replaced with group normalization layers to accommodate private learning.

\textbf{Tasks.} We use two vision  datasets: SVHN \citep{netzer2011reading}  and CIFAR10 \citep{cifar}. SVHN  contains images of $10$ digits and CIFAR10 contains images of 10 classes of real-world objects.

\textbf{Hyperparameters.} We follow the hyperparameters in \citet{zagoruyko2016wide} except using a mini-batch size 1000. This mini-batch size is larger than the default because the averaging effect of large mini-batch reduces the noise variance. The reparametrization rank $r$ is chosen from $\{1, 2, 4, 8, 16\}$. We choose the privacy parameter $\delta<\frac{1}{n}$, and set $\delta=10^{-6}$ for SVHN and $\delta=10^{-5}$ for CIFAR10. We repeat each experiment 3 times and  report the average.

\textbf{Results.} The prediction accuracy with $\epsilon=8$ is presented in Table~\ref{tbl:tbl_resnet}. We can see that RGP (N.P.) achieves comparable performance with training the full model (N.P.). When trained with DP,  RGP outperforms DP-SGD by a considerable margin while enjoying a much lower memory cost. We also compare RGP with DP-SGD using different privacy budgets ($\epsilon=2/4/6$) and report the results  Table~\ref{tbl:vision_vary_eps}.

\subsection{Experiments on the Downstream Tasks of BERT}\label{subsec:exp_bert}

\textbf{Model.} We use the BERT\textsubscript{BASE} model in \citet{devlin2018bert}, which is pre-trained on a massive corpus collected from the Web. The BERT\textsubscript{BASE} model has $\sim$110M parameters.

\textbf{Tasks.} We use four tasks from the General Language Understanding Evaluation (GLUE) benchmark \citep{wang2018glue}, including MNLI, QQP, QNLI, and SST-2. The other  tasks from GLUE are excluded because their datasets are of small sizes (<10K) while differentially private learning  requires  large amount of data \citep{tramer2021differentially}.

\textbf{Hyperparameters.} We follow the hyperparameters in \citet{devlin2018bert}
 except for the mini-batch size and training epochs. The reparametrization rank $r$ is chosen from $\{1, 2, 4, 8\}$. The mini-batch size is 500 for SST-2/QNLI and 1000 for QQP/MNLI. To construct an update with desired mini-batch size, we accumulate the gradients of multiple micro-batches.   We choose $\delta = 10^{-5}$ for QNLI/SST-2 and $\delta =10^{-6}$ for QQP/MNLI. The privacy parameter $\epsilon$ is chosen from $\{1, 2, 4, 6, 8\}$.  The number of training epochs is 50 for $\epsilon>2$ and $20$ for $\epsilon\leq 2$. We run all experiments 5 times with different random seeds and report the average.

\textbf{Results.}   The prediction accuracy of RGP and other baselines is presented in Table~\ref{tbl:tbl_bert}. The results with varying DP parameter $\epsilon$ is plotted in Figure~\ref{fig:fig_bert}.  When trained without privacy guarantee, RGP (N.P.) achieves  test accuracy comparable with fine-tuning the full model.  When trained with differential privacy, RGP achieves the best performance. Its accuracy loss compared to non-private baselines is within $5\%$. The performance of RGP-random is worse than that of RGP because the random subspace does not capture gradient information as effectively as the subspace of historical updates.  DP-SGD achieves the worst performance because high-dimensional noise overwhelms the useful signal in gradients.  We note that DP-SGD runs the lowest because it needs to compute and store 110M floating-point numbers for each individual gradient.

 \begin{figure*}
    \centering
  \includegraphics[width=0.9\linewidth]{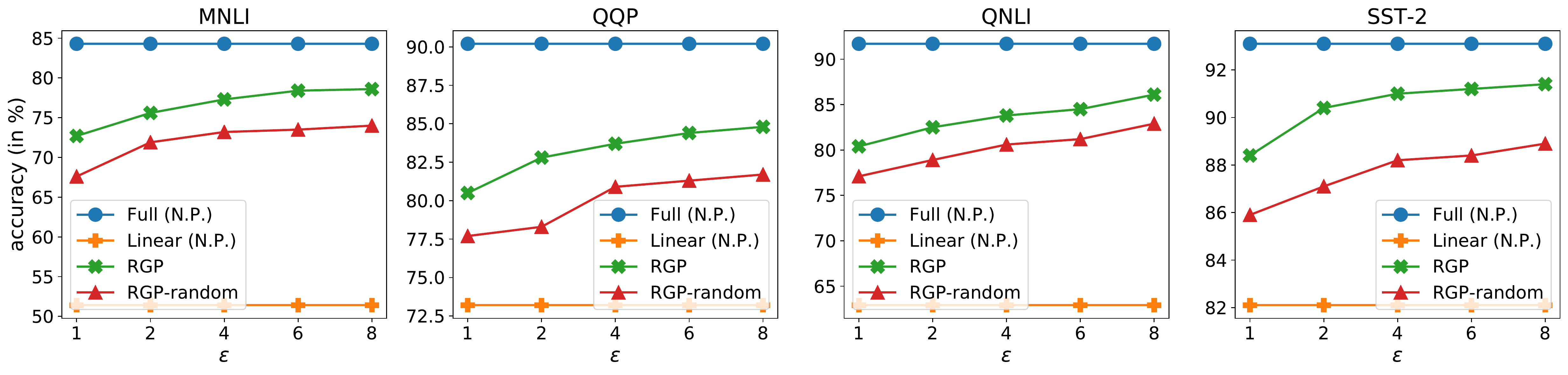}
  \caption{Prediction accuracy of BERT on downstream tasks with varying $\epsilon$. For MNLI, we plot the average score of two test datasets.  }
  \label{fig:fig_bert}
\end{figure*}

\begin{table}
\small
\centering
    \caption{Prediction accuracy of BERT on downstream tasks (in \%). For DP-SGD, RGP, and RGP-random, a same $\epsilon=8$ is used.} 
    
        \begin{tabular}{l|l|l|l|l|l}
        \hline
        \hline
            Method         & MNLI &  QQP  & QNLI  &  SST-2  &  Avg.  \\
            \hline
              Full (N.P.) &     84.8/83.7         &   90.2       &    91.6   &  93.4   &  88.7 \\\cline{1-6}
             Linear (N.P.) &     51.9/50.8         &    73.2     &    63.0     &    82.1   &    64.2     \\\cline{1-6}
                RGP (N.P.)  &  83.6/83.2    &   89.3    &   91.3    &   92.9   & 88.1  \\\cline{1-6}
            DP-SGD\tablefootnote{As shown in \citet{li2021large}, DP-SGD performs better when large batchsizes and full precision are used.}  &     54.6/53.4           &     74.5    &     63.6    &   82.3      &  65.7 \\\cline{1-6}
            RGP-random  &      74.6/73.3        &    81.7   &      82.1 &    87.8    & 79.9 \\\cline{1-6}
               RGP\tablefootnote{The performance of RGP is also better in the above setup. More details are in  \url{https://github.com/dayu11/Differentially-Private-Deep-Learning}. }  &       79.1/78.0       &     84.8    &     86.2  &   91.5  &     83.9
            \\\hline \hline
        \end{tabular}
    \label{tbl:tbl_bert}
\end{table}

\subsection{Defense Against Membership Inference Attack}\label{subsec:exp_mi}

\textbf{Setup.} We use membership inference (MI) attack to empirically evaluate the privacy risk of models trained with/without RGP. Following the membership decision in \citet{sablayrolles2019white}, we predict a sample from the training data if its loss value is smaller than a chosen threshold.  To evaluate the MI success rate, we construct a \emph{MI dataset}, which consists of the same number of training and test samples. Specifically, the MI dataset contains the whole test set and a random subset of the training set. We further divide the MI dataset evenly into two subsets. One is used to find the optimal loss threshold and the other one is used to evaluate the final attack success rate.

\textbf{Results.}
The MI success rates  are presented in Table~\ref{tbl:mi_bert}. For MNLI, QQP, QNLI, and SST-2 datasets, we conduct MI attacks on fine-tuned BERT\textsubscript{BASE} models. For SVHN and CIFAR10 datasets, we conduct MI attacks on trained WRN28-4 models. The MI attack on the models trained with RGP ($\epsilon=8$) is no better than random guessing ($50\%$ success rate), which empirically demonstrate the effectiveness of RGP in protecting privacy. Moreover, interestingly, the models trained with low-rank reparametrization alone also achieve much lower MI success rate than the fully trained model, which indicates the benefit of low-rank reparametrization in terms of privacy protection.

\section{Related Work}

\begin{table}

\renewcommand{\arraystretch}{1.3}
\centering
    \caption{Success rates of membership inference attack  against fine-tuned BERT models (in \%). The closer to 50, the better.} \label{tbl:mi_bert}
    \begin{adjustbox}{max width=0.45\textwidth}
        \begin{tabular}{l|l|l|l|l|l|l}
        \hline
        \hline
            Method         & MNLI &  QQP  & QNLI  &  SST-2   & SVHN & CIFAR10  \\
            \hline
              Full (N.P.) &        60.3      &      56.1    &   55.8   &  57.7  & 56.4 & 58.1  \\\cline{1-7}
                RGP (N.P.)  &      52.3         &   51.5    &    51.8   &   52.6 & 52.8 & 53.3 \\\cline{1-7}
                RGP ($\epsilon=8$)  &      49.9         &   50.0      &   50.4   &   50.1 & 50.1 & 50.3  \\   
           \hline \hline
        \end{tabular}
    \end{adjustbox}
\end{table}

Differentially private learning has a poor dimensional dependency, i.e., the utility degrades dramatically when the model dimension gets large.  In the high-dimensional setting, related works usually assume the sparse structure  \citep{thakurta2013differentially, talwar2015nearly, wang2019sparse, wang2019differentially, cai2019cost} or specific problem structure \cite{chen2020locally,zheng2020locally}. However, these assumptions or specific structures do not hold for the gradient of deep neural networks. Here we emphasize the difference from our low-rank assumption. For the sparsity assumption, the bases are canonical and not private while for the low-rank assumption, it is ``sparse'' under certain bases but the bases are unknown and private. Hence the previous algorithms for sparsity cannot apply here.

Very recently, several works \citep{zhou2020bypassing,kairouz2020dimension, yu2021do} exploit the  redundancy of gradients of samples and suggest projecting the gradients into a low dimensional subspace that is identified by some public data points or historical gradients, in order to reduce the  noise effect when training large models. However, they all require storing and clipping whole individual gradients and hence are hard to train extremely large models.  Our work is orthogonal with theirs, i.e., we exploit the low-rank property of the gradient of each weight matrix, which truly breaks the  barrier of applying DP in large models. 

Another recent approach of training non-convex models with differential privacy is based on the  knowledge transfer of machine learning models \emph{Private Aggregation of Teacher Ensembles (PATE)} \citep{papernot2016semi, papernot2018scalable, jordon2019pate}. They first train independent teacher models on disjoint shards of private data and then tune a student model with privacy by distilling noisy predictions of teacher models on some public samples, whose performance suffers from the data splitting \cite{yu2021do}. It is not clear how to apply PATE to train large language models like BERT. In contrast, our algorithms do not require public data and can be used in different settings with little change.

The phenomenon that the gradients of deep models live on a very low dimensional manifold  has been widely observed \citep{gur2018gradient, vogels2019powersgd, gooneratne2020low, li2020hessian, martin2018implicit, li2018algorithmic}. People have also used this fact to compress the gradient with low-rank approximation  in the distributed optimization scenario \citep{yurtsever2017sketchy, wang2018atomo, karimireddy2019error, vogels2019powersgd}.

\section{Conclusion}
In this paper, we present the reparametrized gradient perturbation (RGP) for applying DP on large models. The key design of RGP exploits two  properties of gradients in  deep neural network, which are 1) the gradient of each weight matrix is of low stable rank, 2) the principal components of historical gradients align well with that of the current gradient.  We also justify the designs with both theoretical and empirical evidence. Thanks to RGP, we are able to train BERT on several downstream tasks with DP guarantee and achieve small accuracy loss. 

\vspace{-1mm}

\section*{Acknowledgements} 

Jian Yin is supported by NSFC (U1711262,  U1711261, U1811264, U1811261, U1911203, U2001211), Guangdong Basic and Applied Basic Research Foundation (2019B1515130001),  Key R\&D Program of Guangdong Province (2018B010107005).

\newpage

\bibliography{general_dl}
\bibliographystyle{icml2021}

\clearpage

\begin{appendix}

\section{Preliminary on Differential Privacy} \label{app:sec:preliminary}

Differential privacy (DP) \cite{dwork2006calibrating,dwork2014algorithmic} is widely recognized as a gold standard of privacy protection due to its mathematical rigor. It controls the maximum influence that any individual sample can produce. The definition of $(\epsilon,\delta)$-DP is given in Definition~\ref{def:dp}.

\begin{definition}[$(\epsilon,\delta)$-DP]
\label{def:dp}
A randomized mechanism $\mathcal{M}$  guarantees $(\epsilon,\delta)$-differential privacy if for any two neighboring input datasets $\sD\sim \sD^{'}$ and for any subset of outputs $\sS$ it holds that $\text{Pr}[\mathcal{M}(\sD)\in \sS]\leq e^{\epsilon}\text{Pr}[\mathcal{M}(\sD^{'})\in \sS]+\delta$.
\end{definition}

Two datasets are said to be neighboring datasets if they only differ in a single sample. When being applied to learning problems, DP requires the  learned models on neighboring datasets have approximately indistinguishable distributions.

\section{Missing Proofs} \label{app:sec:proof}

\subsection{Missing Proofs in Section \ref{sec:lrk}}
\label{apd:subsec:proof_sec2}
\gradlr*
\begin{proof}
The proof is based on the chain rule of back-propagation. Since the reparametrization does not  change the forward and backward signals, we assume the layer inputs are $\sD=\{\vx_{i}\}_{i=1}^{m}$, the corresponding outputs are $\{\vh_i\}_{i=1}^{m}$ with $\vh_i = \mW \vx_i$ and the backward signals on the layer output are $\{\partial \vh_i\}_{i=1}^{m}$. By back-propagation, we have 
\begin{flalign*}
&\partial \mW = \sum_{\vx_{i}\in \sD} (\partial \vh_i) \vx_i^T, \\
&\partial \mL =\sum_{\vx_{i}\in \sD}\partial \vh_i (\mR \vx_i)^T,\;\; \partial \mR =\sum_{\vx_{i}\in \sD} (\mL^{T}\partial \vh_i) \vx_i^T.
\end{flalign*}
Proof is completed by the multiplication associativity.
\end{proof}

\corogradlr*

\begin{proof}
If the columns of $\mL$ and the rows of $\mR$ are orthonormal, the projection of $\partial \mW$ onto  $\mL$ and $\mR$ is defined as,
\begin{flalign}
\mL\mL^T (\partial \mW) + (\partial \mW) \mR^T\mR - \mL\mL^T (\partial \mW)\mR^T\mR. 
\end{flalign}
Substituting the above formula with $\partial \mL = (\partial \mW) \mR^T$ and $\partial \mR = \mL^T (\partial \mW)$ in Theorem \ref{thm:grad_lr}, completes the proof.
\end{proof}

\subsection{Missing Proofs in Section \ref{sec:dp_learning_lrk}}
\label{apd:subsec:proof_sec3}

\privacy*

\begin{proof}
Although RGP releases projected gradient instead of releasing the whole gradient as in \citet{abadi2016deep}, moments accountant is still applicable because it applies to vectorized  function output. 

Moments accountant tracks a bound on the moments of the privacy loss random variable, which is built on the ratio of the probability density functions of the output distributions of two neighboring datasets.  \citet{abadi2016deep} show the log moments of the privacy loss random variable composes linearly. Therefore one can compute the overall privacy cost by adding the log moments at every update. When the training is done, moments accountant  casts the accumulated log moments into $(\epsilon,\delta)$-DP via tail bound. Detailed proof can be found in Appendix B of \citet{abadi2016deep}.

\end{proof}

\subsection{Missing Proofs in Section \ref{sec:grad_property}}
\label{apd:subsec:proof_sec4}
\gradalign*
\begin{proof}
We can compute the gradient at step $t$ 
\begin{flalign*}
\partial \mW_t &= \frac{1}{n}\sum_{i=1}^n (\mW_t \vx_i - \vy_i) \vx_i^T.
\end{flalign*}
Given the gradient descent update \eqref{eq:gd}, we can compute the gradient at $\mW_{t+1}$ as follows
\begin{flalign*}
\partial \mW_{t+1}
&=  \frac{1}{n}\sum_{i=1}^n ((\mW_t - \eta \cdot\partial\mW_t)\vx_i - \vy_i) \vx_i^T\\
&= \frac{1}{n}\sum_{i=1}^n (\mW_t\vx_i - \vy_i) \vx_i^T  - \eta\cdot \partial\mW_t \sum_{i=1}^n \vx_i\vx_i^T \\
& = \partial \mW_t \left (\mI - \eta \sum_{i=1}^n \vx_i\vx_i^T\right ).
\end{flalign*}
Hence we have $\partial \mW_t = \partial \mW_0 \left (\mI - \eta \sum_{i=1}^n \vx_i\vx_i^T\right )^t$. The $\partial \mW_t$ lives in the same subspace for all $t\ge 1$ as they have the same row/column spaces.
\end{proof}

\section{Additional Experiments}\label{app:sec:add-exp}
 \begin{figure*} [t]
    \centering
  \includegraphics[width=0.9\linewidth]{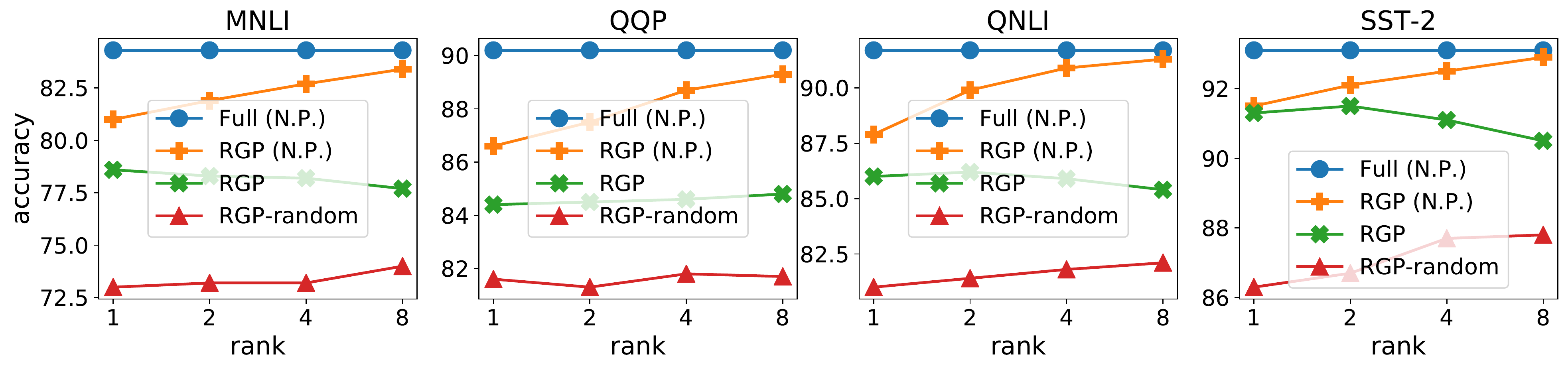}
  \caption{Prediction accuracy of BERT on four downstream tasks (in \%) with difference choices of reparametrization rank. We plot the average score of two test datasets for MNLI. }
  \label{fig:fig_bert_rank}
\end{figure*}
We present some ablation studies in this section to verify the effect of residual weight and reparametrization rank.  In Section~\ref{subsec:apd_rank}, we try RGP with difference rank choices. In Section~\ref{subsec:apd_residual}, we give a variant of RGP that simply discards the residual weight.

\subsection{On the Influence of Different Rank Choices} \label{subsec:apd_rank}

We present the results (see Figure \ref{fig:fig_bert_rank}) with different choices of reparametrization rank.  We consider four algorithms. The first one is fine-tuning the full model that serves as the baseline. The second one is RGP (N.P.) that trains the model with reparametrization but without gradient clipping or adding noise. The third one is RGP (Algorithm~\ref{alg:dp_lrk_repara}) and the  last one is RGP-random, which uses random orthogonal vectors as gradient-carrier matrices. The privacy parameter $\epsilon$ is $8$ and other settings are the same as those in Section~\ref{sec:exp}. The results are plotted in Figure~\ref{fig:fig_bert_rank}.  When the models are trained without noise, increasing the reparametrization rank makes the performance of RGP (N.P.) approach the performance of baseline. When the models are trained with privacy guarantee, increasing the rank sometimes decreases the performance because a larger rank induces more trainable parameters and hence higher noise dimension.

\subsection{On the Importance of Residual Weight}\label{subsec:apd_residual}

Recall that our reparametrization scheme reparametrizes the weight matrix as follows:

\begin{flalign} 
\mW \rightarrow \mL \mR + \tilde{\mW}.{stop\_gradient()}. \label{eq:apd_repara}
\end{flalign}

We have shown that the residual weight $\tilde{\mW}$ keeps the forward/backward signals unchanged and makes the gradients of $\mL$ and $\mR$ naturally connected with the original gradient. To empirically examine the effect of $\tilde{\mW}$, we test the following scheme:

\begin{flalign} 
\mW \rightarrow \mL \mR. \label{eq:apd_repara_nores}
\end{flalign}

 \begin{figure}
    \centering
  \includegraphics[width=0.9\linewidth]{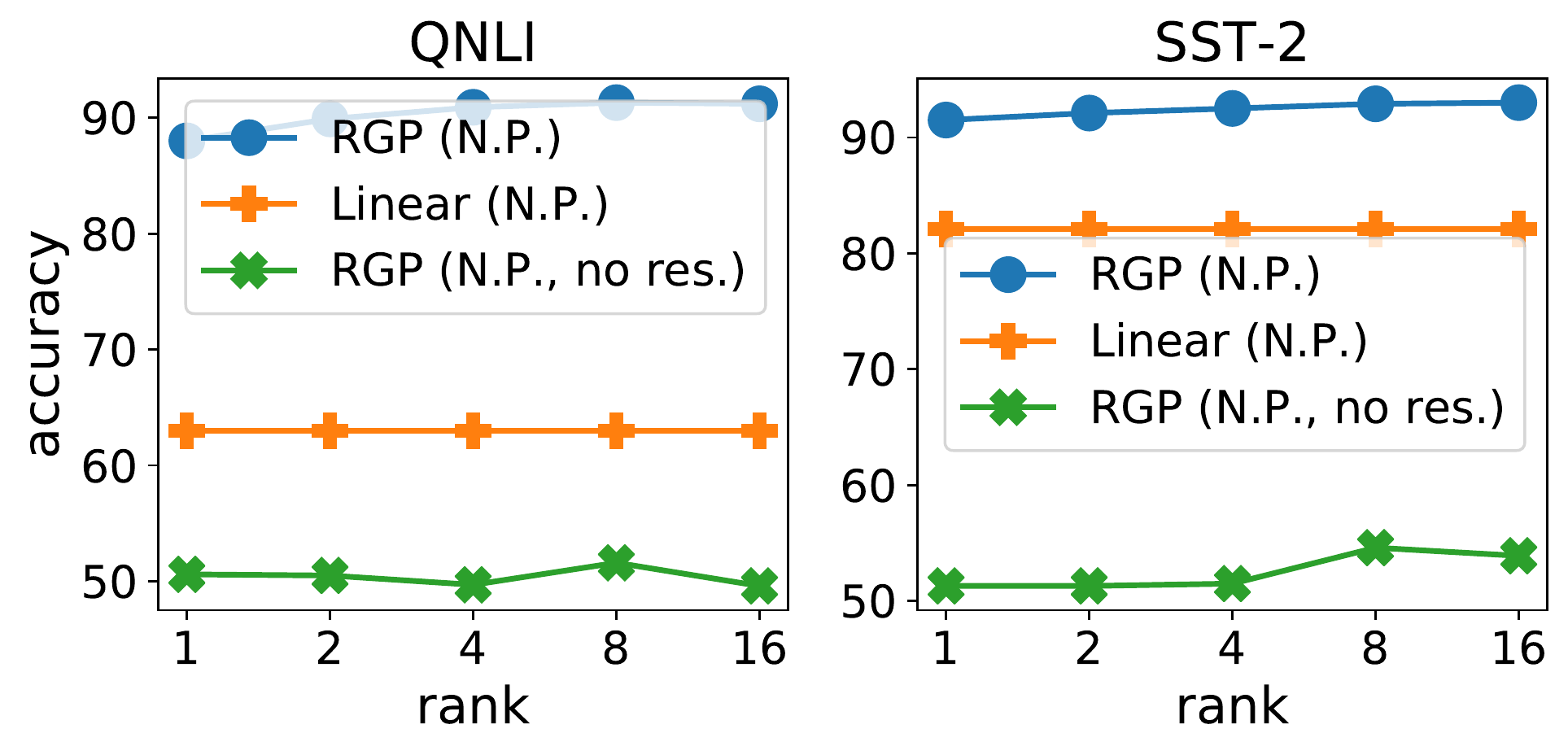}
  \caption{Prediction accuracy of BERT on two downstream tasks. All methods are trained without privacy guarantee.   }
  \label{fig:fig_bert_residual}
\end{figure}

We still use the historical update to generate $\mL$ and $\mR$. Other settings are the same as those in Section~\ref{sec:exp}. The results on two downstream tasks of the BERT model are presented in Figure~\ref{fig:fig_bert_residual}. Without residual weight, the model learns almost nothing from the re-constructed update and the final accuracy is close to the accuracy at initialization.

\end{appendix}

\end{document}